\documentclass[9.5pt]{IEEEtran}

\usepackage[colorlinks,urlcolor=blue,linkcolor=blue,citecolor=blue]{hyperref}

\usepackage{enumitem}
\usepackage{array}
\usepackage[hang]{footmisc}

\usepackage[subrefformat=parens,labelformat=parens,caption=false,font=footnotesize]{subfig}
\usepackage{amsmath}
\usepackage[nolist,nohyperlinks]{acronym}
\usepackage{amsmath}

\begin{acronym}
    \acro{4G}{fourth generation}
    \acro{5G}{fifth generation}
    \acro{AoA}{angle of arrival}
   	\acro{AoD}{angle of departure}
   	\acro{AP}{access point}
    \acro{BCRLB}{Bayesian CRLB}
    \acro{BG}{beam group}
    \acro{BS}{base stations}
    \acro{BPP}{binomial point process}
    \acro{BP}{broadcast probing}
	\acro{CDF}{cumulative density function}
    \acro{CCDF}{complementary cumulative density function}
    \acro{CF}{closed-form}
    \acro{CKM}{constrained $K$-means}
    \acro{DAS}{distributed antenna system}
    \acro{CRLB}{Cramer-Rao lower bound}
    \acro{ECDF}{empirical cumulative distribution function}
    \acro{EI}{Exponential Integral}
    \acro{eMBB}{enhanced mobile broadband}
    \acro{FIM}{Fisher Information Matrix}
    \acro{GoF}{goodness-of-fit}
    \acro{GPS}{global positioning system}
    \acro{GE}{group exploration}
    \acro{GNSS}{global navigation satellite system}
    \acro{HetNets}{heterogeneous networks}
    \acro{IoT}{internet of things}
    \acro{IIoT}{industrial internet of things}
    \acro{ILP}{integer linear program}
    \acro{KM}{$K$-means}
    \acro{KPI}{key performance indicator}
    \acro{KHM}{$K$-harmonic means}
    \acro{WKHM}{weighted $K$-harmonic means}
    \acro{KC}{$K$-centers}
    \acro{LOS}{line of sight}
    \acro{LLR}{log-likelihood ratio}
    \acro{MAB}{multi-armed bandit}
    \acro{MBS}{macro base station}
    \acro{MD}{\ac{MD}}
    \acro{MEC}{mobile-edge computing}
    \acro{mIoT}{massive internet of things}
    \acro{MIMO}{multiple input multiple output}
    \acro{mm-wave}{millimeter wave}
    \acro{mMTC}{massive machine-type communications}
    \acro{MS}{mobile station}
    \acro{MVUE}{minimum-variance unbiased estimator}
    \acro{NLOS}{non line-of-sight}
    \acro{OFDM}{orthogonal frequency division multiplexing}
    \acro{PAC}{probably approximately correct}
    \acro{PDF}{probability density function}
    \acro{PGF}{probability generating functional}
    \acro{PLCP}{Poisson line Cox process}
    \acro{PLT}{Poisson line tessellation}
    \acro{PLP}{Poisson line process}
    \acro{PPP}{Poisson point process}
    \acro{PV}{Poisson-Voronoi}
    \acro{QoS}{quality of service}
    \acro{RAT}{radio access technique}
    \acro{RIC}{radio access network intelligent controller}
    \acro{RL}{reinforcement-learning}
    \acro{RN}{radio node}
    \acro{RSSI}{received signal-strength indicator}
    \acro{RSRP}{reference signal received power}
    \acro{BS}{base station}
    \acro{SINR}{signal to interference plus noise ratio}
    \acro{SG}{stochastic geometry}
    \acro{SNR}{signal to noise ratio}
    \acro{SSB}{synchronization signal block}
    \acro{SWIPT}{simultaneous wireless information and power transfer}
    \acro{TS}{Thompson Sampling}
    \acro{TS-CD}{TS with change-detection}
    \acro{KS}{Kolmogorov-Smirnov}
    \acro{UCB}{upper confidence bound}
	\acro{ULA}{uniform linear array}
    \acro{UPA}{uniform planar array}
	\acro{UE}{user equipment}
 	\acro{URLLC}{ultra-reliable low-latency communications}
    \acro{V2V}{vehicle-to-vehicle}    
    \acro{wpt}{wireless power transfer}
\end{acronym}

\usepackage{graphicx,wrapfig,lipsum}
\usepackage{rotating}
\usepackage{amsmath, amssymb, amsthm}
\usepackage{stmaryrd}
\usepackage{tikz}
\usepackage{verbatim}
\usepackage{url}
\usepackage[utf8]{inputenc}
\usepackage[english]{babel}
\newtheorem{theorem}{Theorem}[]

\newtheorem{lemma}[]{Lemma}

\usepackage{multicol}
\newtheorem{rmk}{Remark}[]

\usepackage{algorithm}
\usepackage{algpseudocode}
\usepackage{scalerel}
\usepackage{multicol}
\usepackage{setspace}

\usepackage[noadjust]{cite}
\usepackage{booktabs}
\usepackage{derivative}
\usepackage{bbm}
\usepackage[normalem]{ulem}
\definecolor{Ao}{rgb}{0.0, 0.62, 0.1}
\usepackage{color}
\usepackage{dsfont}
\usepackage{bm}
\usepackage{setspace}
\usepackage{cellspace}
\usepackage{array}
\usetikzlibrary{automata}
\usetikzlibrary{arrows}
\usetikzlibrary{shapes.geometric, arrows}
\usepackage[]{footmisc}
\usetikzlibrary{positioning}
\usetikzlibrary{calc}
\usetikzlibrary{arrows}
\hypersetup{nolinks=true}
\allowdisplaybreaks

\usepackage{caption}
\usepackage{textcomp}

\usepackage{mathtools}
\begin{document}
	\title{\fontsize{22.8}{27.6}\selectfont{Weighted $K$-Harmonic Means Clustering: Convergence Analysis and Applications to Wireless Communications}
    }
	 \author{Gourab Ghatak \vspace{-1cm}
  \thanks{G. Ghatak is with the Department of Electrical Engineering, IIT Delhi, India 110016; Email: gghatak@ee.iitd.ac.in.
   }}
  
\date{}

\maketitle

\begin{abstract}
We propose the \emph{weighted K-harmonic means} (WKHM) clustering algorithm, a regularized variant of K-harmonic means designed to ensure numerical stability while enabling soft assignments through inverse-distance weighting. Unlike classical K-means and constrained K-means, WKHM admits a direct interpretation in wireless networks: its weights are exactly equivalent to fractional user association based on received signal strength. We establish rigorous convergence guarantees under both deterministic and stochastic settings, addressing key technical challenges arising from non-convexity and random initialization. Specifically, we prove monotone descent to a local minimum under fixed initialization, convergence in probability under Binomial Point Process (BPP) initialization, and almost sure convergence under mild decay conditions. These results provide the first stochastic convergence guarantees for harmonic-mean-based clustering. Finally, through extensive simulations with diverse user distributions, we show that WKHM achieves a superior tradeoff between minimum signal strength and load fairness compared to classical and modern clustering baselines, making it a principled tool for joint radio node placement and user association in wireless networks.
\end{abstract}


\section{Introduction}
The design of clustering algorithms that are robust, differentiable, and adaptable to uncertain environments is critical in signal processing, machine learning and wireless communications. The work by Ezugwu {\it et al.}~\cite{ezugwu2022comprehensive} provides a comprehensive survey and recent advances in clustering algorithms. In wireless communications, clustering finds applications in several contexts, e.g., clustering of base stations to enable joint-transmission in cell-free \ac{MIMO}~\cite{le2021learning}, clustering of \acp{UE} to serve via non-orthogonal multiple access transmission~\cite{cui2018unsupervised}, clustering of \ac{UE} to perform wide-area beamforming~\cite{dai2025clustering}, and clustering of transmitters to enable cognitive-radio services~\cite{dai2016clustering}. 

The classical \ac{KM} clustering algorithm (see~\cite{likas2003global}) minimizes the sum of the squared Euclidean distance from the data points $\{x_i\}_{i = 1}^N \subset \mathbb{R}^d$ to their nearest centroids $\{m_l\}_{l = 1}^K \subset \mathbb{R}^d$. Mathematically, the \ac{KM} metric is represented as
\begin{align}
    \mathcal{L}_{\rm KM} = \sum_{i = 1}^N \min_l\{||x_i - m_l||^2\}, \label{eq:KM metric}
\end{align}
which is to be minimized. For uniformly located data points, \ac{KM} intuitively appears to be a reasonable method to designate the centroid locations. However, for a non-homogeneous data point density, it is well-known that \ac{KM} may result in imbalanced cluster allotments. To alleviate the above issue, \cite{bradley2000constrained} had proposed the \ac{CKM} algorithm. \ac{CKM} and its variants impose an upper and/or a lower limit on the number $(N_l)$ of data points that associate to each centroid $m_l$. Formally, \ac{CKM} minimizes \eqref{eq:KM metric} subject to
\begin{align}
    N_{\rm min} \leq  N_l \leq N_{\max} \label{eq:capacity constraint}. 
\end{align}
where $0 \leq N_{\min} \leq N_{\max} \leq N$. However, such a capacity constraint albeit relevant for several applications, is unsuitable for many wireless network realizations as will be shown in this work. This is due to the fact that in order to meet the capacity constraint, \ac{CKM} tends to create disproportionately sized clusters, i.e., larger cluster area in regions with lower data point density while smaller cluster area in regions with higher density.

In a wireless \ac{UE} association context, such a clustering association scheme results in a higher number of cell-edge \acp{UE} of large cells thereby deteriorating their signal strength. Although \ac{KM} and \ac{CKM} are intuitive algorithms and are simple to implement, for finite \ac{UE} locations, they do not optimize the downlink \ac{UE} signals.
    {
\begin{lemma}
\label{lem:not_sumRSRP_optimal}
Let $\mathcal U=\{u_1,\dots,u_N\}$ be a finite set of UE positions in $\mathbb R^d$ and let $\mathcal R=\{r_1,\dots,r_K\}$ denote the positions of $K$ \acp{RN} or cluster centers produced by the KM or \ac{CKM} algorithm. Assume a (single-tone) \ac{RSRP} model for a UE at position $x$ served by an RN at position $r$ of the form
$
\mathrm{RSRP}(x,r) \;=\; \frac{P}{\big(\|x-r\|^\alpha + \delta\big)},
\qquad \alpha>0,\ \delta\ge 0,\ P>0.
$
Then, for finite $N$, the \ac{KM} and \ac{CKM} solutions are not in general guaranteed to maximize the sum-RSRP objective
$ \sum_{u\in\mathcal U} \max_{r\in\mathcal R}\mathrm{RSRP}(u,r).$ In particular, there exist UE configurations for which the KM centroid placement (or its constrained variant) does not maximize sum-RSRP.
\end{lemma}
}
{
\begin{proof}
We prove the statement by counterexample (this suffices because the lemma is a negative/existence claim). Consider the one-dimensional case $d=1$ and $K=1$ (a single RN) and let us place two UEs at $u_1 = -\tfrac{d}{2}$ and  $u_2 = +\tfrac{d}{2}$, with separation $d>0$. Let the RN position be a scalar $r\in\mathbb R$. The sum-RSRP when the single RN is at $r$ equals
$S(r) \;=\; \frac{P}{\big(|r+ \tfrac{d}{2}|^\alpha + \delta\big)} \;+\; \frac{P}{\big(|r- \tfrac{d}{2}|^\alpha + \delta\big)}$.

The standard KM objective with squared Euclidean distance for a single cluster center $r$ is
\[
J_{\text{KM}}(r) \;=\; \sum_{i=1}^2 \|u_i - r\|^2 \;=\; (r+\tfrac{d}{2})^2 + (r-\tfrac{d}{2})^2.
\]
Differentiating and setting derivative to zero yields the unique minimizer
$r_{\text{KM}} = 0,$ i.e. the midpoint of the two UEs. Accordingly, $S(0) \;=\; \frac{P}{\big((d/2)^\alpha + \delta\big)} \;+\; \frac{P}{\big((d/2)^\alpha + \delta\big)}
= \frac{2P}{(d/2)^\alpha + \delta}$.

Now if we consider $S(r)$ at $r=u_2 = +d/2$ (placing RN exactly at UE $u_2$): $S\big(\tfrac{d}{2}\big) \;=\; \frac{P}{\big(d^\alpha + \delta\big)} \;+\; \frac{P}{\delta}$.
Consequently, the differential improvement in the sum-RSRP is $\Delta \;:=\; S\big(\tfrac{d}{2}\big) - S(0)
= \frac{P}{\delta} + \frac{P}{d^\alpha + \delta} - \frac{2P}{(d/2)^\alpha + \delta} > 0$. Hence placing the RN at the UE location $u_2$ yields strictly larger sum-RSRP than the KM centroid $r_{\text{KM}}=0$.
\end{proof}
}
In this regard, the \ac{KHM} algorithm, which facilitates inverse weights to the Euclidean distances is an attractive choice. Specifically, the \ac{KHM} metric as defined in~\cite{zhang1999k, zhang2000generalized} is $
    \mathcal{L}_{\rm KHM} = \sum_{i = 1}^{N}\frac{K}{\sum_{l = 1}^{K} \frac{1}{||{\bf x}_i - {\bf m_l}||^p}}$.
Note that the \ac{KHM} metric considers the sum of the \ac{RSRP} from all the \acp{RN} to each \ac{UE}. In case the \acp{RN} are not facilitating joint transmission such as in cell-free \ac{MIMO}, and form separate cells, this formulation thus is not an accurate representation of the downlink \ac{UE} performance. In fact the $\mathcal{M}_{\rm KHM}$ formulation is counter-productive since it considers the interfering signals as useful signals. 
In order to take advantage of the features of both \ac{KM} and \ac{KHM}, we propose the \ac{WKHM} algorithm by incorporating soft assignments through inverse-distance weighting.
\subsection{Key Contributions}
\begin{itemize}
    \item \textbf{Novel WKHM formulation:} We introduce the \ac{WKHM} algorithm with a regularized distortion function that ensures numerical stability and avoids degeneracy. Unlike prior extensions of K-means or K-harmonic means, our formulation admits a direct interpretation as a soft user association rule in wireless networks, where the WKHM weights are mathematically equivalent to received-signal-strength-based association fractions.
    
    \item \textbf{Rigorous convergence guarantees:} We provide the first convergence analysis of harmonic-mean-based clustering under both deterministic and stochastic settings. Our results establish monotone descent to a local minimum under fixed initialization, convergence in probability under \ac{BPP} initialization, and almost sure convergence under mild decay conditions. These proofs address non-trivial challenges of non-convexity, uniform convergence, and probabilistic stability not handled by classical K-means or fuzzy C-means analysis.
    
    \item \textbf{Wireless interpretability:} We show that WKHM naturally bridges clustering theory and wireless association models. Through Monte Carlo simulations under diverse user distributions, we demonstrate that WKHM achieves a superior balance between minimum user signal strength and load fairness compared to classical clustering baselines. {Ablation studies further highlight the necessity of regularization in attaining these gains.}
\end{itemize}

\section{Regularized WKHM Formulation}

Let $ \mathcal{X} = \{x_1, \dots, x_N\} \subset \mathbb{R}^d $ be a fixed dataset. Let $ \{m_1, \dots, m_K\} \subset \mathbb{R}^d $ denote cluster centers. We define the regularized WKHM distortion as
\begin{equation}
\mathcal{L} = \sum_{i=1}^N \sum_{l=1}^K \frac{w_{il}}{(\|x_i - m_l\|^2 + \epsilon)^{p/2}},
\label{eq:loss}
\end{equation}
where $ \epsilon > 0 $ ensures numerical stability, and $ p > 0 $. The soft assignment weights are defined as:
\begin{equation}
w_{il} = \frac{(\|x_i - m_l\|^2 + \epsilon)^{-q/2}}{\sum_{j=1}^{K} (\|x_i - m_j\|^2 + \epsilon)^{-q/2}},
\label{eq:weights}
\end{equation}
for some sharpness parameter $ q > 0 $. {    The gradient with respect to the center $m_\ell$ is
\[
\nabla_{m_\ell}\mathcal{L}
= \sum_{i=1}^N w_{i\ell}\,(x_i - m_\ell)\,
\Bigg[(p+q)\,C_{i\ell} - q\,\frac{\displaystyle\sum_{j=1}^K w_{ij}C_{ij}}{d_{i\ell}^2+\epsilon}\Bigg].
\]

Equating $\nabla_{m_\ell}\mathcal{L}=0$ and rearranging yields the update:
\[
m_\ell^{(t+1)} = \frac{\displaystyle\sum_{i=1}^N B_{i\ell}^{(t)}\,x_i}
{\displaystyle\sum_{i=1}^N B_{i\ell}^{(t)}},
\]
where the per-sample update weights are given by
\[
B_{i\ell}^{(t)} = w_{i\ell}^{(t)}\,
\Bigg[(p+q)\,C_{i\ell}^{(t)} - q\,\frac{\displaystyle\sum_{j=1}^K w_{ij}^{(t)}C_{ij}^{(t)}}{d_{i\ell}^{2(t)}+\epsilon}\Bigg].
\]

Here $d_{ij}^{2(t)}=\|x_i-m_j^{(t)}\|^2$, $C_{ij}^{(t)}=(d_{ij}^{2(t)}+\epsilon)^{-p/2}$, $A_{ij}^{(t)}=(d_{ij}^{2(t)}+\epsilon)^{-q/2}$, and $w_{ij}^{(t)}=A_{ij}^{(t)}/\sum_{r}A_{ir}^{(t)}$.}


{{\bf Choice of the smoothing parameter $\epsilon$:}
The additive tuning parameter $\epsilon>0$ in serves primarily as a numerical and probabilistic regularizer: it prevents division by very small values and ensures stable inversion. Its choice must balance bias (introduced by the regularization) against variance (reduced by stabilizing small denominators). In the common case where the denominator converges in probability to a positive constant $B>0$, any sequence $\epsilon_n\to 0$ preserves consistency; choosing $\epsilon_n=o(1/\sqrt{n})$ also preserves the $\sqrt{n}$-asymptotic distribution. If, however, the denominator may approach zero (or fluctuate at scale $\tau_n$, for example $\tau_n\asymp n^{-1/2}$), then $\epsilon_n$ should be chosen to dominate those fluctuations to avoid large variance, e.g. $\epsilon_n\gg \tau_n$ (a pragmatic choice is $\epsilon_n=n^{-1/4}$ when $\tau_n\asymp n^{-1/2}$). In practice we recommend the following simple rules: (i) when empirical evidence indicates the denominator is bounded away from zero, we can use $\epsilon=1/n$ (or $\epsilon=10^{-6}$ for moderate $n$), (ii) when small denominators can occur we can use $\epsilon=n^{-1/4}$.}
We first consider the case where the centers $ m_l $ are initialized deterministically and fixed across iterations except through the update rule.

\begin{theorem}[Convergence to a Local Minimum]
Let $ \mathcal{X} \subset \mathbb{R}^d $ be a fixed, finite dataset, and let $ \{m_l^{(0)}\}_{l=1}^K $ be an initial set of distinct cluster centers. Suppose $ p, q > 0 $ and regularization $ \epsilon > 0 $ are fixed. Then the regularized WKHM algorithm generates a sequence $ \{\mathcal{L}^{(t)}\} $ such that
\begin{align}
    \mathcal{L}^{(t+1)} \leq \mathcal{L}^{(t)}, \quad \text{and} \quad \lim_{t \to \infty} \mathcal{L}^{(t)} = \mathcal{L}^{(\infty)},
\end{align}
where $ \mathcal{L}^{(\infty)} $ corresponds to a local minimum of the distortion function.
\end{theorem}
\begin{proof}
Let us consider the loss function at iteration $ t $, given the current cluster centers $ \{m_l^{(t)}\} $. The loss function is:
\begin{align}
\mathcal{L}^{(t)} = \sum_{i=1}^{N} \sum_{l=1}^{K} \frac{w_{il}^{(t)}}{(\|x_i - m_l^{(t)}\|^2 + \epsilon)^{p/2}},
\end{align}
where the weights are $w_{il}^{(t)} = \frac{(\|x_i - m_l^{(t)}\|^2 + \epsilon)^{-q/2}}{\sum_{j=1}^{K} (\|x_i - m_j^{(t)}\|^2 + \epsilon)^{-q/2}}$. In each iteration, the algorithm performs two steps: (1) compute $ w_{il}^{(t)} $, and (2) update each $ m_l^{(t+1)} $ using the weighted average, i.e.,
$m_l^{(t+1)} = \frac{\sum_{i=1}^{N} B_{il}^{(t)} x_i}{\sum_{i=1}^{N} B_{il}^{(t)}}$, where $ B_{il}^{(t)} $ are derived from the gradient of $ \mathcal{L} $ with respect to $ m_l $. Since the updates minimize a surrogate of $ \mathcal{L} $, and because $ \mathcal{L}^{(t)} \geq 0 $ for all $ t $, the sequence $ \{\mathcal{L}^{(t)}\} $ is non-increasing and bounded below, and hence converges, i.e., $\mathcal{L}^{(t)} \to \mathcal{L}^{(\infty)}$. The limit point satisfies the fixed-point update condition. Since each update step corresponds to solving $ \nabla_{m_l} \mathcal{L} = 0 $, the gradient of $ \mathcal{L} $ with respect to each $ m_l $ vanishes at convergence, implying a stationary point. Because of the non-convex nature of the problem, this stationary point is a local minimum.
\end{proof}

We now consider a stochastic setting where the centers $ \{m_l\} $ are drawn from a Binomial Point Process (BPP). Let $ \Phi = \{m_1, \dots, m_K\} \sim \text{BPP}(K, \Lambda) $, where $ \Lambda \subset \mathbb{R}^d $ is compact. The initial distortion function $ \mathcal{L}^{(t)}(\Phi) $ becomes a random variable over the space of realizations.
\begin{theorem}[Convergence in Probability]
Let $ \Phi = \{m_1, \dots, m_K\} \sim \text{BPP}(K, \Lambda) $, where $ \Lambda \subset \mathbb{R}^d $ is compact. Let $ \mathcal{L}^{(t)}(\Phi) $ be the distortion value after $ t $ iterations of the WKHM algorithm with initial centers $ \Phi $. Then for any $ \delta > 0 $ and $ \eta > 0 $, there exists $ T \in \mathbb{N} $ such that:
\begin{align}
\mathbb{P}\left( \left| \mathcal{L}^{(t+1)}(\Phi) - \mathcal{L}^{(t)}(\Phi) \right| > \delta \right) < \eta, \quad \forall t > T.
\end{align}
Thus, $ \mathcal{L}^{(t)}(\Phi) \xrightarrow{\mathbb{P}} \mathcal{L}^{(\infty)}(\Phi) $.
\end{theorem}

\begin{proof}
For any fixed realization $ \Phi \in \Lambda^K $, Theorem 1 applies and guarantees that the sequence $ \mathcal{L}^{(t)}(\Phi) $ is non-increasing and converges to a finite limit. Since $ \Lambda $ is compact, all $ \Phi \in \Lambda^K $ lie in a bounded space, and each $ \mathcal{L}^{(t)} $ is a continuous function of the centers $ \{m_l\} $. Let us define
\begin{align}
f_t(\Phi) = \left| \mathcal{L}^{(t+1)}(\Phi) - \mathcal{L}^{(t)}(\Phi) \right|.
\end{align}
Because $ \mathcal{L}^{(t)} $ is continuous in $ \Phi $, and the updates are Lipschitz in $ m_l $ (due to the regularization $ \epsilon > 0 $), the functions $ f_t $ are measurable~\cite{heinonen2005lectures}. Moreover, for every $ \delta > 0 $, since $ f_t(\Phi) \to 0 $ for each $ \Phi $, by Egorov's theorem, the convergence is uniform outside a set of small measure~\cite{farris1983generalization}. Hence, for any $ \eta > 0 $, there exists $ T $ such that for all $ t > T $, $\mathbb{P}(f_t(\Phi) > \delta) < \eta$. This implies convergence in probability.
\end{proof}

\begin{theorem}[Almost Sure Convergence]
Under the setting of Theorem 2, assume additionally that there exists a constant $ \alpha > 0 $ such that for all $ \Phi $, $f_t(\Phi) \leq C e^{-\alpha t}$ for some constant $ C > 0 $. Then $\mathcal{L}^{(t)}(\Phi) \xrightarrow{\text{a.s.}} \mathcal{L}^{(\infty)}(\Phi)$.
\end{theorem}

\begin{proof}
Define the event
\begin{align}
A_\delta^t := \left\{ \Phi : \left| \mathcal{L}^{(t+1)}(\Phi) - \mathcal{L}^{(t)}(\Phi) \right| > \delta \right\}.
\end{align}
By the exponential decay assumption, $\mathbb{P}(A_\delta^t) \leq \frac{C}{\delta} e^{-\alpha t}$.
The sum over $ t $ is finite, i.e., $\sum_{t=1}^\infty \mathbb{P}(A_\delta^t) < \infty$.
Hence, by the Borel–Cantelli lemma~\cite{chung1952application},
\begin{align}
\mathbb{P}\left( A_\delta^t \text{ occurs infinitely often} \right) = 0.
\end{align}
This holds for all $ \delta > 0 $, which in turn implies
\begin{align}
\mathbb{P} \left( \left| \mathcal{L}^{(t+1)}(\Phi) - \mathcal{L}^{(t)}(\Phi) \right| \to 0 \right) = 1.   
\end{align}
Combined with monotonicity and boundedness of $ \mathcal{L}^{(t)}(\Phi) $, this yields convergence almost surely.
\end{proof}
{
\begin{rmk}
The exponential decay condition $f_t(\Phi) \le C e^{-\alpha t}$, $\alpha>0$, assumed in the original version of this theorem, is a sufficient but not necessary special case. 
In contrast, a condition $\sum_{t=1}^\infty f_t(\Phi) < \infty$ allows slower decays such as $f_t(\Phi) = O(t^{-(1+\delta)})$, $\delta>0$, thereby extending the applicability of the Theorem to a broader class of stochastic perturbations, including sub-exponential and polynomially decaying noise.
\end{rmk}
}

\section{Application to Wireless Communications: Node Placement and Signal-Fraction Association}
A key network operation is determining association rules between the \acp{UE} and the \acp{BS}. Current industry-standard rules follow \ac{RSSI} based association, which is often suboptimal in terms of the \ac{SINR}. \ac{RSSI}-based association, in case of homogeneous propagation environment results in an equivalent rule of nearest-base-station association. Such rules are insufficient when interference and topology evolve stochastically.

Consider a downlink wireless network consisting of $K$ base stations (BSs) located at positions $\{m_l\}_{l=1}^K \subset \mathbb{R}^2$. Let $\mathcal{X} = \{x_1, \dots, x_N\} \subset \mathbb{R}^2$ denote the positions of $N$ mobile users. The BSs are assumed to be deployed either deterministically or as a realization of a Binomial Point Process (BPP), representing randomly placed small cells or distributed antennas. Each BS transmits with equal power $P_0$, and the signal from BS $l$ received at user $i$ experiences path loss governed by the standard power-law model:
\begin{equation}
P_{il} = \frac{P_0}{(\|x_i - m_l\|^2 + \epsilon)^{\gamma/2}},
\label{eq:rx_power}
\end{equation}
where $\gamma > 2$ is the path-loss exponent and $\epsilon > 0$ is a regularization constant that accounts for minimum path loss due to hardware constraints or distance floors.


The power-based soft user association corresponds exactly to WKHM clustering with $q = \gamma$:
\[
w_{il}^{\text{WKHM}} = \frac{(\|x_i - m_l\|^2 + \epsilon)^{-\gamma/2}}{\sum_{j=1}^K (\|x_i - m_j\|^2 + \epsilon)^{-\gamma/2}}.
\]
Thus, solving a WKHM clustering problem over the user and BS locations is equivalent to computing soft user association based on received signal strengths.

First we demonstrate how the clustering with \ac{WKHM} dramatically changes with the number of data points and then we discuss its implications in wireless networks. We compare the cluster formations with $K = 3$ for \ac{WKHM}, \ac{KM}, along with fuzzy C-means clustering (FCM)~\cite{bezdek1984fcm}, and clustering based on Gaussian mixture models (GMM)~\cite{zhang2021gaussian}. Fig.~\ref{fig:cluster1} shows that for $N = 240$, all the algorithms form the same clusters. However, for $N = 330$, we see a considerable difference between the other algorithms and \ac{WKHM}. In order to have a balance between the distance to the {\it connected} centroid and other other centroids, the \ac{WKHM} displaces one centroid to the near proximity to the other.

\begin{figure}
    \centering
    \includegraphics[width=\linewidth, height = 0.5\linewidth]{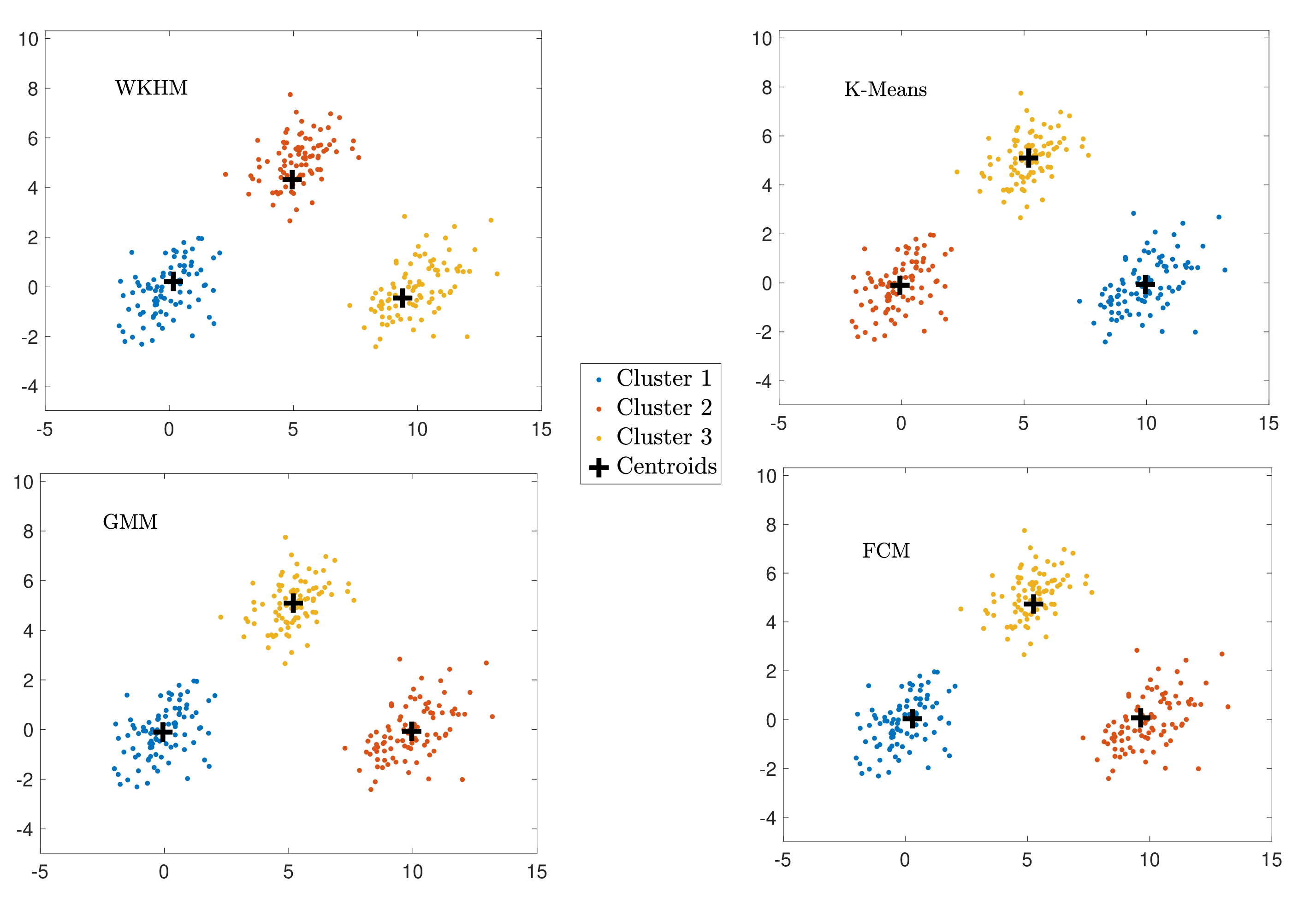}
    \caption{Comparison of cluster formations with $K = 3$ and $N = 240$.}
    \label{fig:cluster1}
\end{figure}

\begin{figure}
    \centering
    \includegraphics[width=\linewidth, height = 0.5\linewidth]{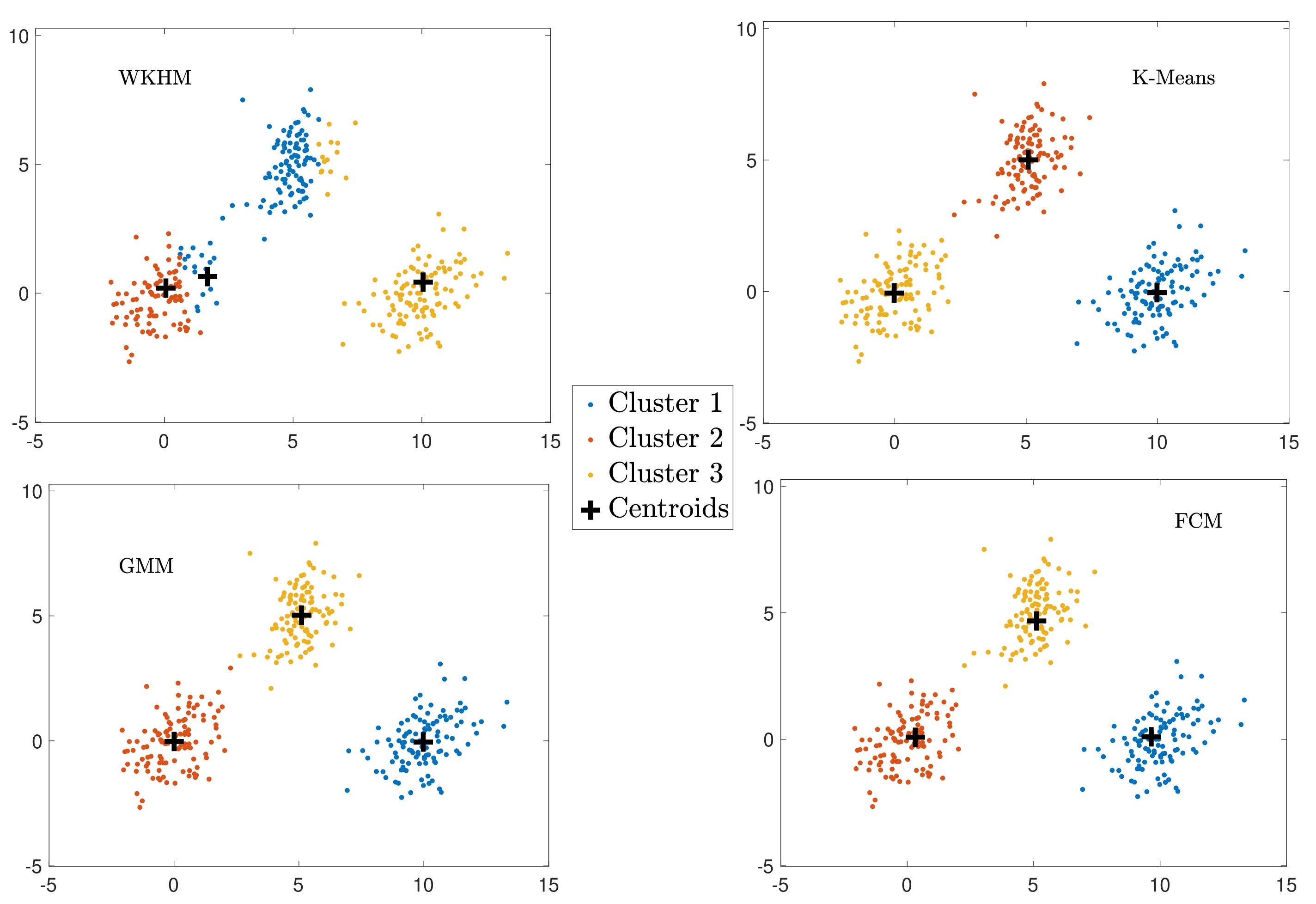}
    \caption{Comparison of cluster formations with $K = 3$ and $N = 330$.}
    \label{fig:cluster2}
\end{figure}

To elaborate on the implications of the above, we employ the algorithms for the joint \ac{RN} placement and \ac{UE} association problem. Apart from the \ac{KM}, \ac{CKM}, and \ac{KHM} algorithms we also investigate the performance of \ac{KC} for max-min fairness. In case the network designer intends to guarantee for the worst case \ac{UE} \ac{RSRP} experience, the \ac{KC} algorithm can be an attractive option. The \ac{KC} algorithm attempts to maximize the minimum distance between the RN centers and the \acp{UE}. The corresponding metric is
\begin{align}
    \mathcal{L}_{\rm KC} = \max_i \{ \min_l \{ ||x_i - m_l||^2\}\}.
    \label{eq:KC_metric}
\end{align}
Unfortunately, the \ac{KC} clustering problem is NP-HARD, and thus, it is unlikely that there can ever be efficient polynomial time exact algorithms. However, sub-optimal approximate algorithms exist for the problem. In our simulations, we use the algorithm presented in \cite{kleindessner2019fair} for solving the KC problem.



\begin{figure*}
   \subfloat[]{\includegraphics[width=0.33\linewidth, height = 0.2\linewidth]{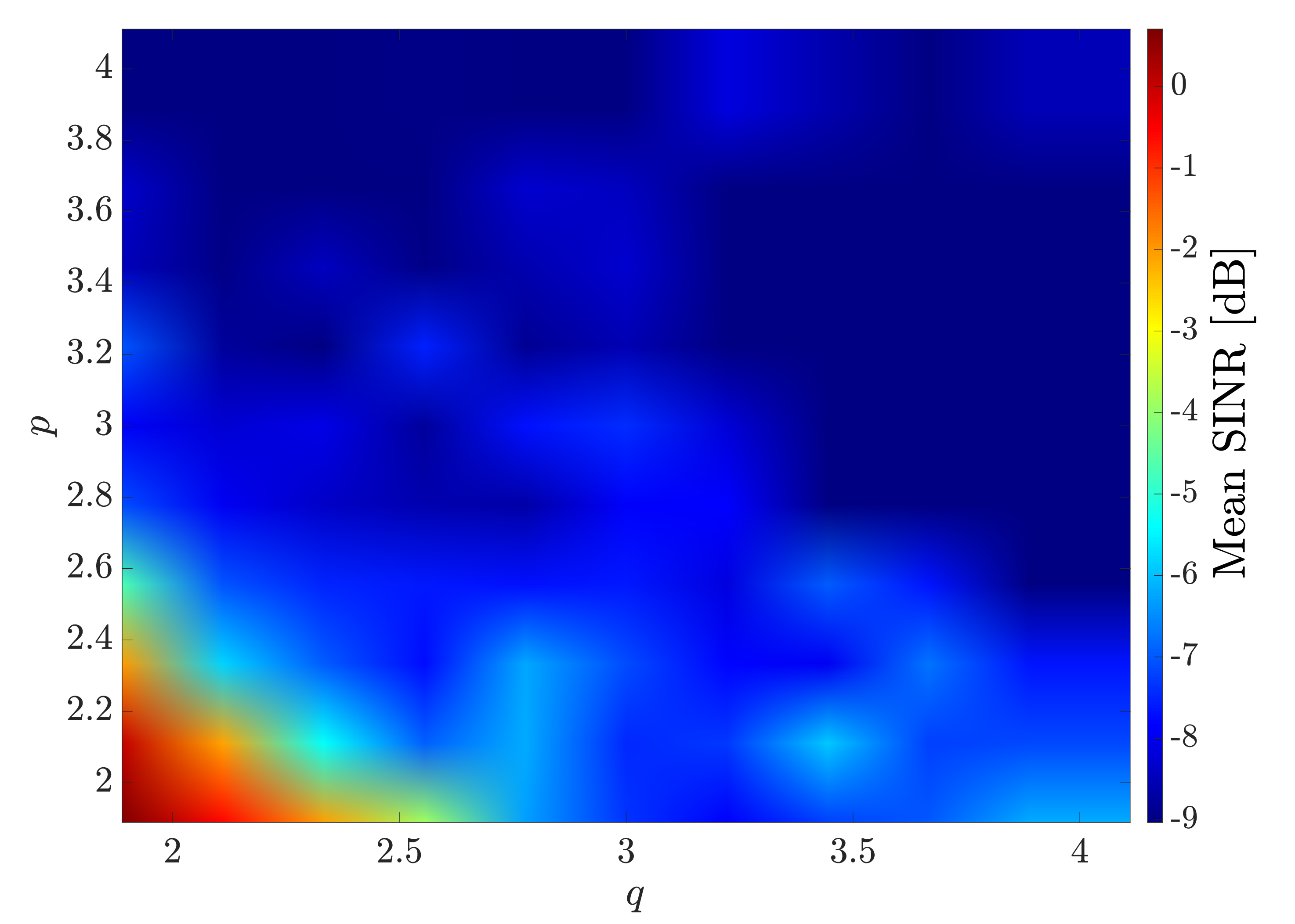}\label{fig:Mean_SINR_WCL}}
   \hfil
   \subfloat[]{\includegraphics[width=0.33\linewidth, height = 0.2\linewidth]{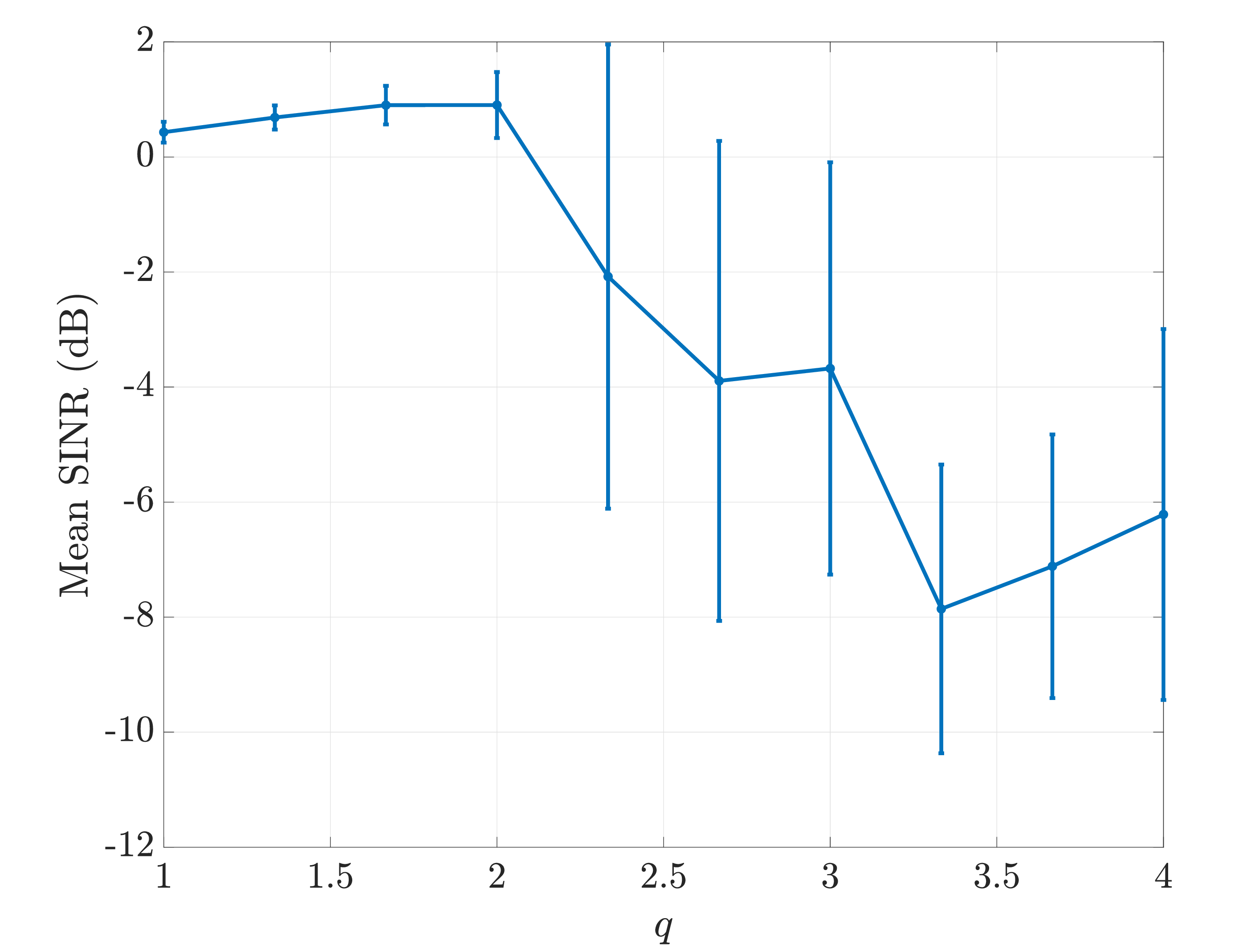}\label{fig:Abalation_q}}
   \hfil
   \subfloat[]{\includegraphics[width=0.33\linewidth, height = 0.2\linewidth]{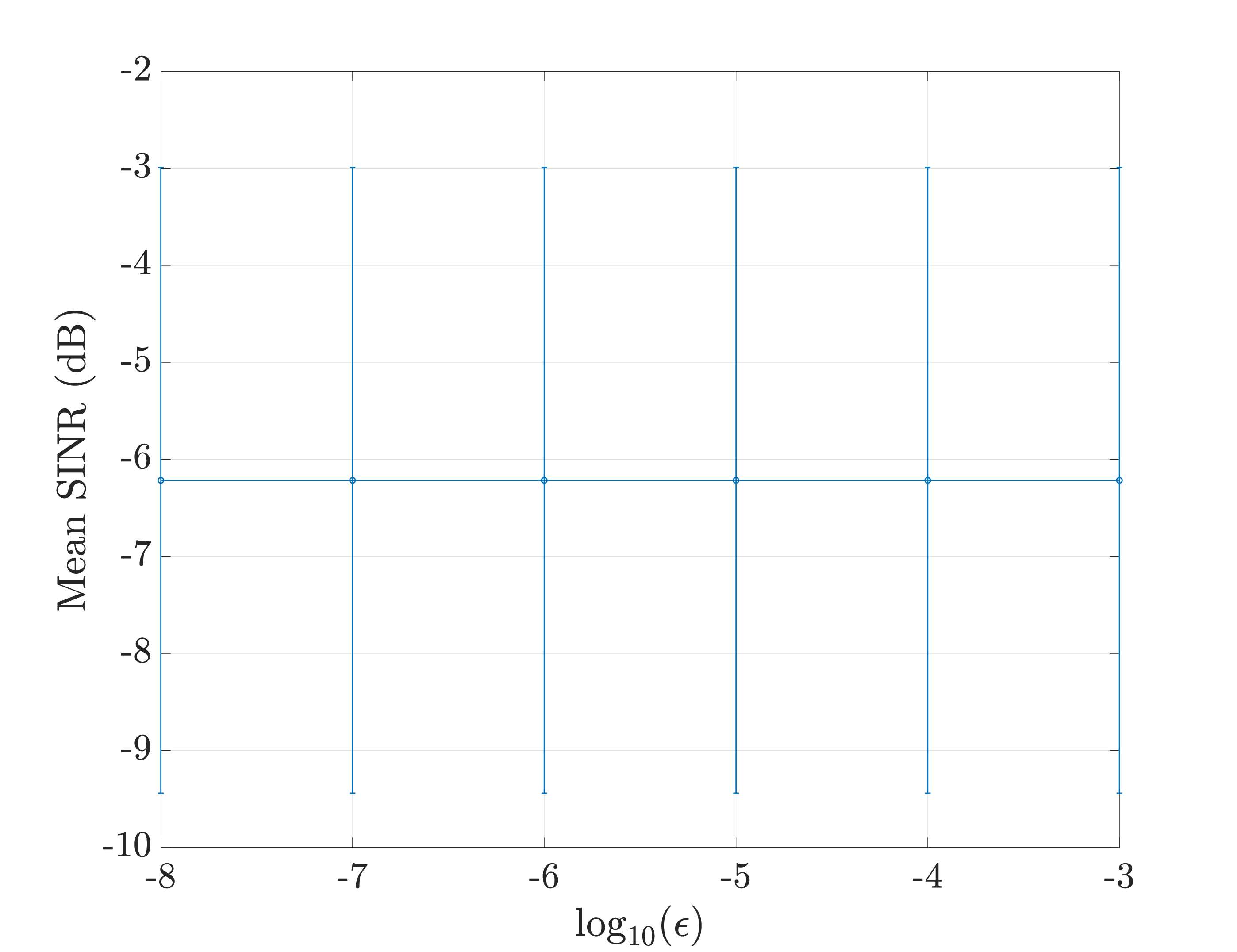}\label{fig:Epsilon_WCL}}
   \caption{Variation of performance with respect to $p$, $q$, and $\epsilon$: (a) SINR as a function of $p$ and $q$, (b) Mean SINR with respect to $q$ for a fixed $p = 2$, and (c) Mean SINR with respect to $\epsilon$ for fixed $p = 2$ and $q = 4$.}
   \label{fig:Ablation}
\end{figure*}
 {   Fig.~\ref{fig:Super} presents the empirical CDFs of three important user experience metrics - SNR, SINR, and cell load. Fig.~\ref{fig:SNR_WCL} shows that KM and KHM exhibit slightly higher SNR values in the upper tail; however, this improvement does not translate into higher SINR due to increased inter-cell interference near cell edges. As seen in Fig.~\ref{fig:SINR_WCL}, both KM and KHM produce heavier lower tails in the SINR distribution, indicating poorer performance for edge users. In contrast, WKHM's SINR CDF is consistently shifted rightward, implying fewer users experience low-SINR conditions and overall improved signal fidelity. It is important to note that the KC algorithm performs the best for the cell-center users (or high-SINR users) as compared to the WKHM, due to its worst user-centric clustering strategy leading to favorable SINR conditions for cell center users (as new centroids are placed away from them).
 The advantage of WKHM becomes more pronounced when examining the load distribution in Fig.~\ref{fig:Load_WCL}. While KC attains a favorable SINR profile, it leads to highly uneven cell loads, with several RNs serving a disproportionate number of UEs. KM and CKM exhibit the opposite behavior-uniform load distribution but relatively low SINR. 

 To elaborate on this further, {recall that the only source of randomness in the SNR/SINR statistics arises from the spatial position of the receiver, while fast fading is assumed to be averaged through control-channel measurements. Hence, the probabilistic success $\mathbb{P}(\mathrm{SINR}>\theta)$ is based solely on spatial variability.
Here, a key distinction from typical stochastic-geometry based works, e.g.,~\cite{5200994} is that our RR-based throughput model incorporates the effect of load through the clustering structure, which is essential in indoor deployments but not included in the throughput definition used in~\cite{5200994}.  

The throughput improvements under WKHM result from the combined effect of (i) improved worst-case SINR 
and (ii) more balanced load across RNs. From Fig.~\ref{fig:SINR_WCL}, WKHM achieves $\mathbb{P}(\mathrm{SINR}>\theta)=1$ at $\theta=0$ dB, whereas the other methods require approximately $\theta=-5$ dB. Additionally, Fig.~\ref{fig:Load_WCL} shows that KC exhibits severe load imbalance, with some cells serving more than 35 UEs, while WKHM keeps all cells between 9 and 13 users.   Since RR throughput scales inversely with load, these two effects jointly explain the higher throughput 
observed for the WKHM scheme.}

 
To substantiate this, in Table~\ref{tab:ttest}, WKHM demonstrates significantly higher mean throughput compared to KM ($p < 10^{-4}$), CKM ($p < 10^{-3}$), and KC ($p < 10^{-2}$), while also outperforming KHM with a marginal but consistent gain. The negative $t$-statistics indicate that the throughput of WKHM exceeds that of the other schemes across the realizations. These results confirm that the improved balance between signal strength (SNR), interference suppression (SINR), and load distribution translates into a statistically significant enhancement in user throughput. Consequently, WKHM achieves a superior trade-off among control-channel reliability, interference robustness, and fairness in resource allocation, leading to the best overall user experience.}

\begin{figure*}
   \subfloat[]{\includegraphics[width=0.33\linewidth]{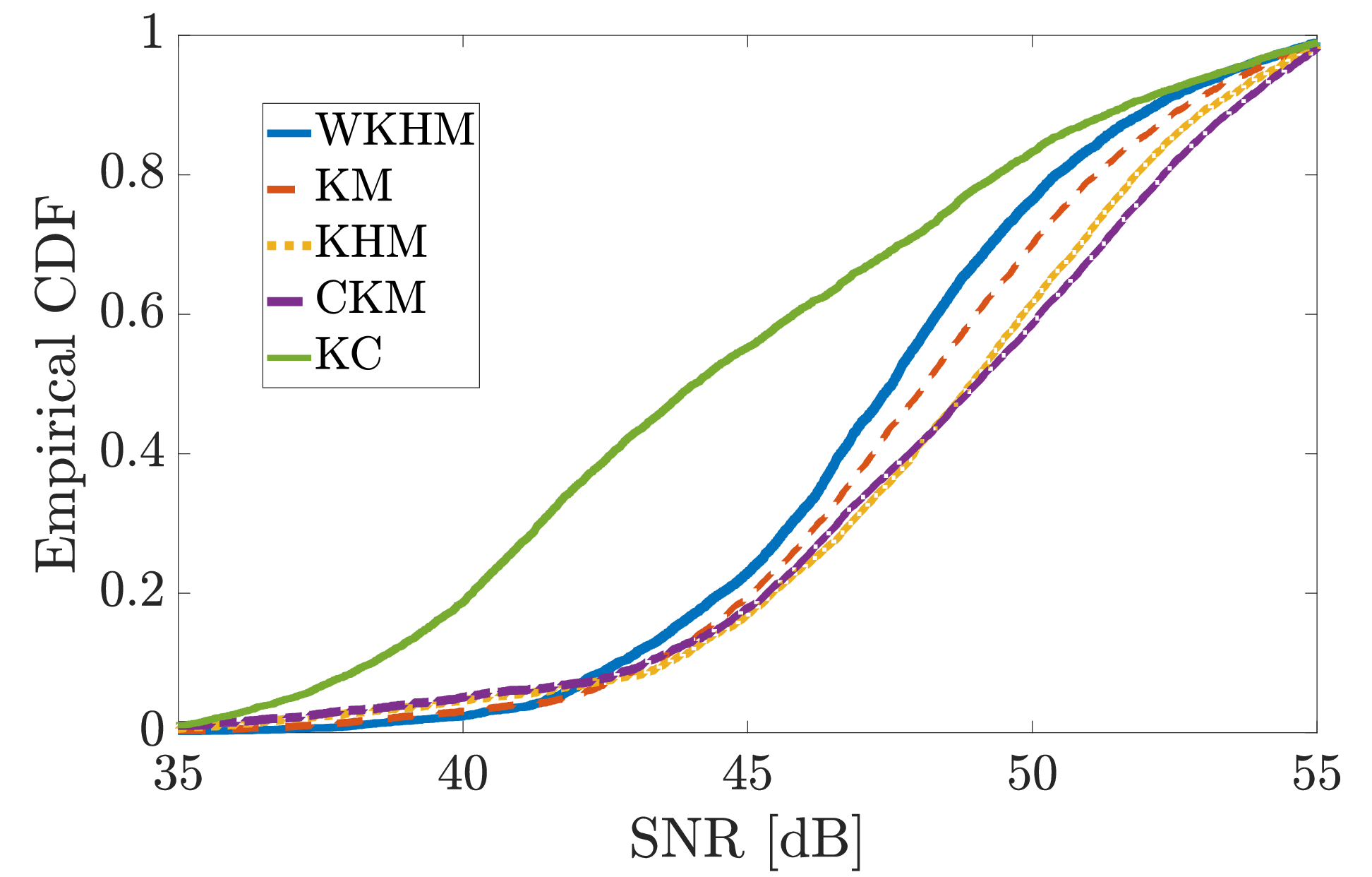}\label{fig:SNR_WCL}}
   \hfil
   \subfloat[]{\includegraphics[width=0.33\linewidth]{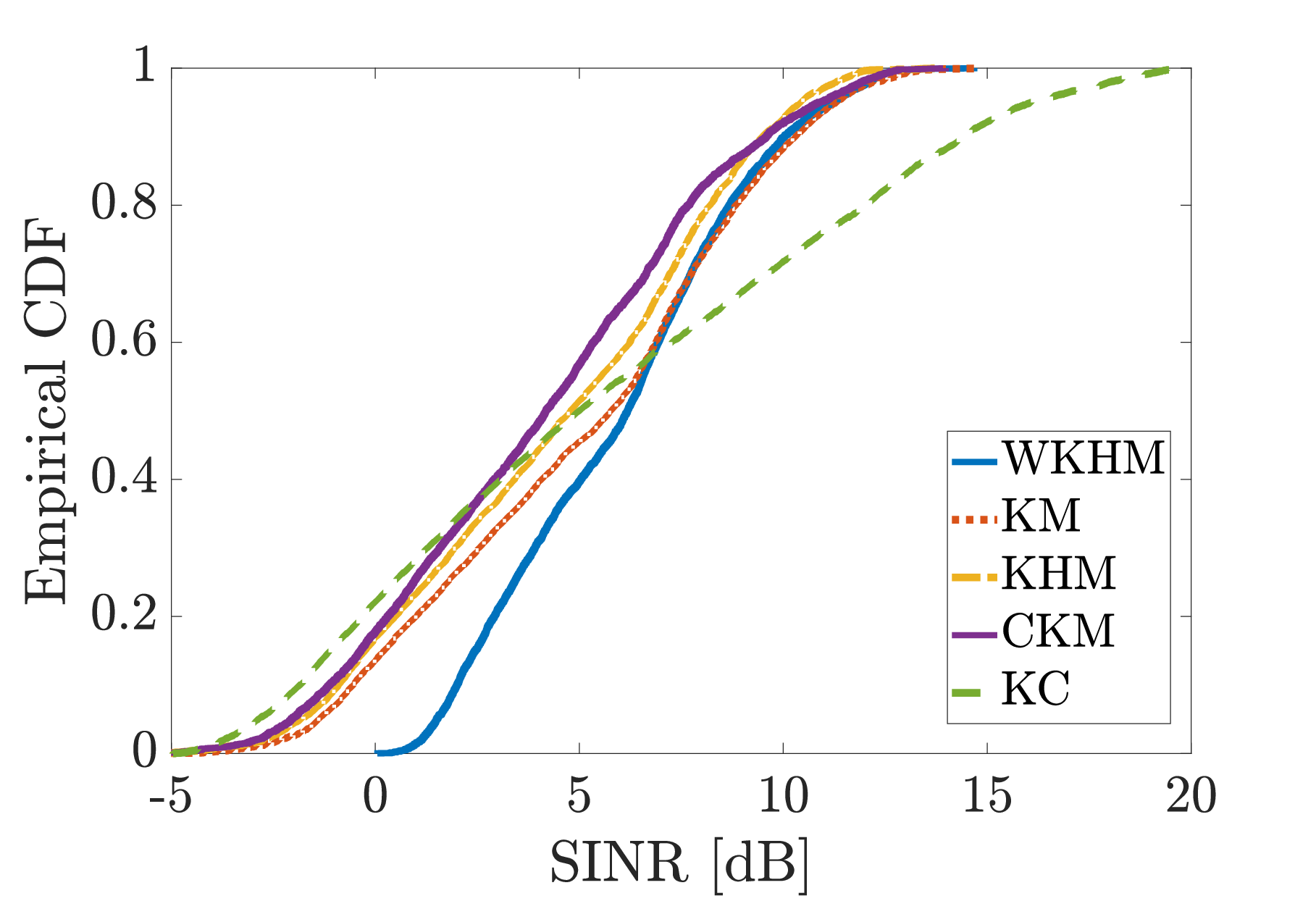}\label{fig:SINR_WCL}}
   \hfil
   \subfloat[]{\includegraphics[width=0.33\linewidth, height = 0.22\linewidth]{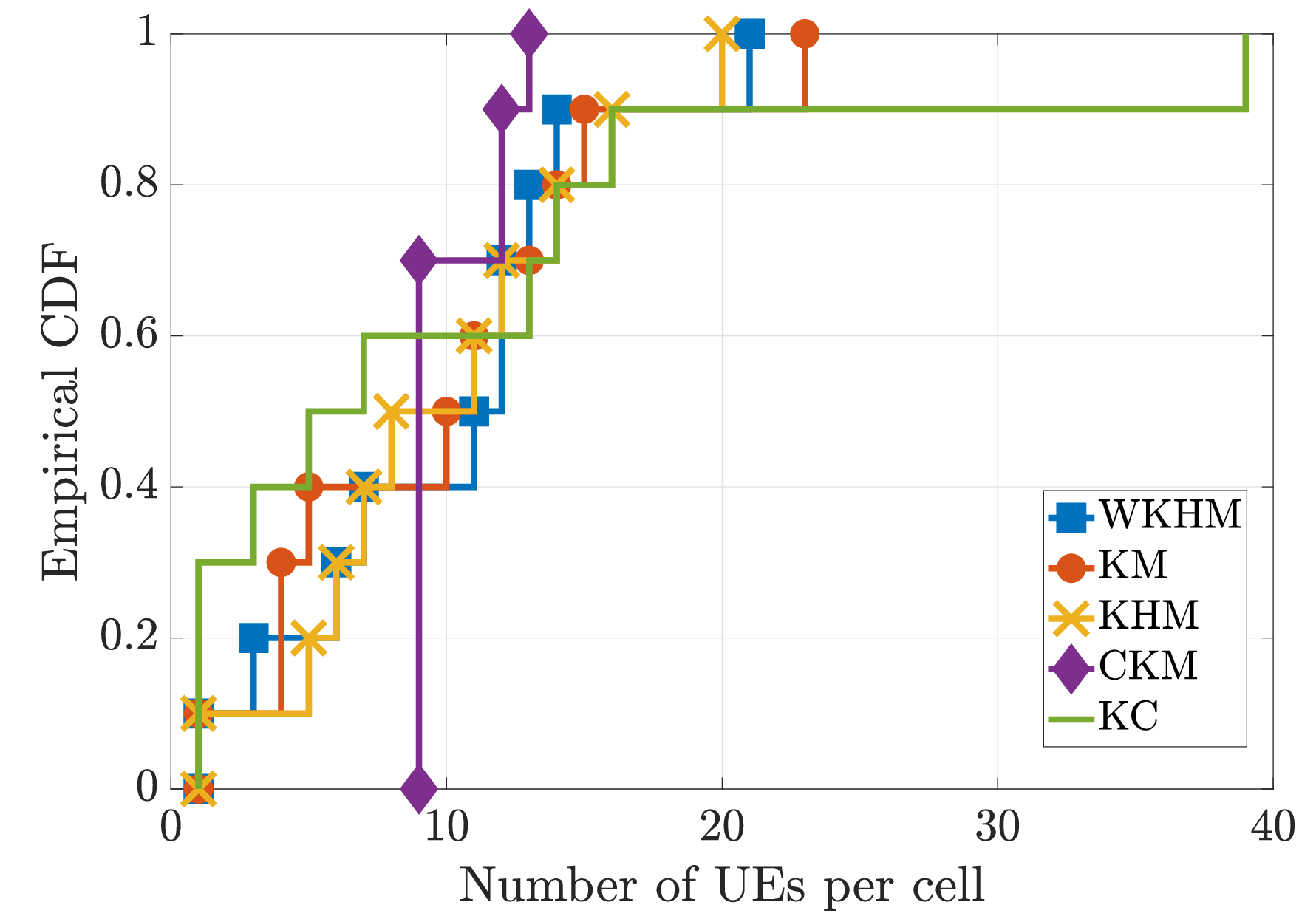}\label{fig:Load_WCL}}
   \caption{Empirical CDFs (a) SNR, (b) SINR, and (c) Cell load for different clustering schemes. Here the number of base stations is 9 in a 100 m x 100 m area with 100 users.}
    \label{fig:Super}
\end{figure*}

\begin{table}[t]
\centering
\begin{tabular}{lcc}
\hline
\textbf{Algorithm Pair} & \textbf{$t$-statistic} & \textbf{$p$-value} \\
\hline
KM vs WKHM   & $-4.72$ & $8.5\times10^{-5}$ \\
CKM vs WKHM  & $-3.14$ & $3.8\times10^{-3}$ \\
KHM vs WKHM  & $-2.01$ & $5.4\times10^{-2}$  \\
KC vs WKHM   & $-2.87$ & $7.1\times10^{-3}$ \\
\hline
\end{tabular}
\caption{Paired $t$-test comparison of mean per-user throughput between WKHM and other clustering schemes (two-tailed, $\alpha = 0.05$). Values are in Mbps.}
\label{tab:ttest}
\vspace{-0.3cm}
\end{table}

{{\bf Ablation Study:}
    Fig.~\ref{fig:Mean_SINR_WCL} shows that the mean SINR stabilizes and then deteriorates beyond $p \approx 2.5$, suggesting that the useful signal becomes too feeble as compared to the interference power. Fig.~\ref{fig:Abalation_q} shows that varying $q$ affects the sharpness of the weights: the performance improves as $q$ increases up to $2$, beyond which it deteriorates. This indicates that WKHM is sensitive to exact value of $q$ and accordingly, requires fine-tuning. This is intuitive since the purpose for the weights in WKHM is to emulate the signal fraction, which is best emulated with $q \approx p$. Fig.~\ref{fig:Epsilon_WCL} demonstrates that the performance remains nearly constant for $\epsilon \in [10^{-7},10^{-5}]$, confirming numerical stability in this range.
    }

\section{Conclusion}

We presented a comprehensive analysis of the regularized Weighted K-Harmonic Means (WKHM) algorithm that extends classical clustering approaches by incorporating soft assignment through inverse-distance weighting, along with regularization to ensure numerical stability. We rigorously proved that the algorithm converges to a local minimum under deterministic initialization of cluster centers. When the cluster centers are modeled as realizations of a Binomial Point Process (BPP), a common abstraction in stochastic geometry, we established convergence of the distortion sequence in probability and, with mild decay conditions on the per-iteration objective differences, also proved almost sure convergence. These results provide strong theoretical guarantees for the use of WKHM in random or data-driven environments. 

\bibliography{references.bib}
\bibliographystyle{ieeetr}

\end{document}